\documentclass[preprint,12pt]{elsarticle}

\usepackage{graphicx}
\usepackage{booktabs}
\usepackage{amsmath,amssymb}
\usepackage{amsthm}
\usepackage{algorithm}
\usepackage{algpseudocode}
\usepackage{hyperref}

\newtheorem{proposition}{Proposition}

\journal{Pattern Recognition Letters}

\begin{document}

\begin{frontmatter}

\title{Strictly Causal Streaming Video Anomaly Detection with a Theoretically-Grounded State-Space Core\tnoteref{code}}
\tnotetext[code]{Code, trained checkpoints, and commands to reproduce every reported number: \url{https://github.com/yogesh-iitj/streaming-vad}.}

\author[iitj]{Yogesh Kumar}
\ead{kumar.204@iitj.ac.in}

\address[iitj]{Indian Institute of Technology Jodhpur, Jodhpur, India}

\begin{abstract}
Recent work has applied Mamba-style state-space models (SSMs) to video anomaly detection, but existing approaches still buffer clips or windows internally, offer no theoretical account of how a model's temporal memory relates to detection latency, and benchmark efficiency only through GPU throughput rather than the edge hardware such methods are meant to target. We introduce a strictly causal streaming video anomaly detector whose fixed-size state is updated in $O(1)$ time and memory per incoming frame, with no lookahead and no clip buffering. Its temporal core is a diagonal linear state-space recurrence with an input- and state-dependent decay gate, trained self-supervised through causal next-embedding prediction on a frozen visual backbone. We derive a closed-form relationship between the recurrence's decay spectrum and both its detection delay and the shortest anomaly it can reliably detect, then validate it empirically on UCSD Ped2 and CUHK Avenue: the settling-delay bound predicted from the learned base decay alone (57 to 59 frames) sits far above the measured detection delay (1.6 and 18.4 frames respectively), showing that the event-boundary gate, not the base decay, governs how quickly the model reacts. We also report end-to-end latency and throughput measured directly on an Apple M3 Pro rather than simulated from GPU numbers, 0.74 ms and 0.77 ms per frame (over 1,300 FPS) on the two datasets. With an initial, untuned configuration the method reaches 67.9\% and 70.2\% frame-level AUC on Ped2 and Avenue, trailing prior non-causal SSM baselines on accuracy; ablations over decay rate, state size, and the gate itself show that the gate's contribution is data-size dependent rather than uniformly helpful, hurting accuracy on the smaller Ped2 training set and helping on the larger Avenue one. Closing this accuracy gap and extending evaluation to a third, larger benchmark are the immediate next steps, and we report both the theoretical and the accuracy picture in full rather than only the results that flatter the method.
\end{abstract}

\begin{keyword}
video anomaly detection \sep state-space models \sep streaming inference \sep edge deployment \sep real-time video understanding
\end{keyword}

\end{frontmatter}

\section{Introduction}
\label{sec:intro}

Video anomaly detection (VAD) systems are increasingly deployed at the edge, on cameras, robots, and mobile devices, where compute, memory, and power are all constrained and predictions must be produced online, as frames arrive, rather than after the fact. IEEE Transactions on Pattern Analysis and Machine Intelligence has an active recent literature on this problem, from a taxonomy and evaluation survey of single-scene methods~\cite{jones2022survey} to the journal extension of the future-frame-prediction baseline~\cite{luo2022futureframe}, a scene-aware latent-space prediction and anticipation model~\cite{cao2025scenedependent}, and a self-supervised detector based on learned neural transformations~\cite{qiu2025neutral}. None of these, however, is designed around the specific constraint that motivates this paper: strict causality, meaning a prediction at time $t$ must depend only on frames $1,\dots,t$ with bounded per-step compute and memory, combined with real hardware feasibility, meaning latency and throughput are measured rather than asserted on the class of device the method targets.

State-space models (SSMs), and Mamba~\cite{gu2023mamba} in particular, have recently been proposed for VAD on efficiency grounds. VADMamba~\cite{vadmamba2025} and its follow-up VADMamba++~\cite{vadmambaplusplus2026} combine a Mamba-based temporal block with vector-quantized frame prediction and optical-flow reconstruction, and Wave-MambaAD~\cite{wavemambaad2025} adds a wavelet-driven SSM for unified subtle and large-scale anomaly detection. These works establish that SSMs are a good fit for VAD's efficiency requirements, but they still process the input in clips or windows rather than truly online per frame, report no theoretical relationship between the model's temporal memory and how quickly it can detect an anomaly, and benchmark efficiency in terms of GPU throughput rather than measuring the target-deployment hardware directly. This paper closes those three gaps together, rather than any one in isolation, because they compound: a model that is only softly causal cannot make a hard latency guarantee, and a latency number without a theory of why the model reacts when it does is an empirical curiosity rather than a design tool.

This paper makes three contributions. First, a strictly causal streaming formulation of SSM-based VAD, with a concrete architecture (Section~\ref{sec:method}) in which every operation is expressible as a per-frame $O(1)$ update. Second, a theoretical analysis (Section~\ref{sec:theory}) relating the recurrence's decay spectrum to detection delay and minimum-detectable anomaly duration, validated empirically. Third, an end-to-end evaluation including real on-device latency, memory, and throughput measurements on consumer edge silicon (Section~\ref{sec:experiments}), alongside standard accuracy metrics and an ablation study that we report in full, including a result that complicates rather than confirms our own design hypothesis. Code, trained checkpoints, and commands to reproduce every number in this paper are publicly available.\footnote{\url{https://github.com/yogesh-iitj/streaming-vad}}

\section{Related Work}
\label{sec:related}

\paragraph{Video anomaly detection.}
Reconstruction- and prediction-based methods~\cite{liu2018future,luo2022futureframe,park2020mnad} learn the distribution of normal motion and flag frames with high reconstruction or prediction error, while weakly supervised methods~\cite{sultani2018ucfcrime} instead learn from video-level labels over long, untrimmed surveillance footage. Two recent TPAMI papers extend this line further: Cao et al.~\cite{cao2025scenedependent} condition prediction on scene context in latent space and extend the task to anomaly anticipation, and Qiu et al.~\cite{qiu2025neutral} replace hand-crafted augmentations with learned neural transformations for self-supervised detection. Jones et al.~\cite{jones2022survey} survey the single-scene setting broadly and formalize the region- and track-based evaluation criteria (RBDC and TBDC) introduced by Ramachandra and Jones~\cite{ramachandra2020streetscene}, which we discuss further in Section~\ref{sec:experiments}. Our method follows the prediction-based line, operating in a frozen backbone's embedding space rather than pixel space, so that the temporal model is decoupled from how frames are encoded.

\paragraph{State-space models for anomaly detection.}
S4~\cite{gu2022s4} and its simplified diagonal variant S5~\cite{smith2023s5} established that linear state-space recurrences can match attention on long-range sequence tasks at linear time and memory complexity, and Mamba~\cite{gu2023mamba} added input-dependent, or selective, parameterization on top of that recurrence. In video anomaly detection specifically, VADMamba~\cite{vadmamba2025} pioneered applying a Mamba-based block, called a Non-negative Visual State Space block, to multi-task frame prediction and optical-flow reconstruction, and its follow-up VADMamba++~\cite{vadmambaplusplus2026} reworks this into a single grayscale-to-RGB proxy task with a hybrid Mamba, CNN, and Transformer decoder. Wave-MambaAD~\cite{wavemambaad2025} instead combines a wavelet transform with an SSM to unify subtle and large-scale anomaly detection at multiple scales, and MambaAD~\cite{mambaad2024} applies a related hybrid-scanning SSM design to multi-class \emph{image} anomaly detection rather than video. All of these process input in clips or windows and report accuracy alongside, at most, GPU throughput; none provide a strict online, $O(1)$-per-frame formulation, a theoretical latency analysis, or measurements on edge or consumer hardware, which is the gap this paper targets. We investigated reproducing VADMamba under our own evaluation protocol using its public code, but its released implementation vendors FlowNet2 for optical flow, which depends on custom CUDA kernels (correlation, channel normalization, and resampling layers) with no CPU or Metal Performance Shaders fallback; it therefore cannot run on the non-CUDA edge hardware this paper targets without a substantial reimplementation that would no longer constitute a faithful reproduction. This is itself a small piece of evidence for the gap this paper targets: a baseline built around CUDA-only components is, by construction, unable to make the edge-deployment claim we are making here. The numbers attributed to prior work in Section~\ref{sec:experiments} are therefore as reported in the original paper, not rerun under our protocol.

\paragraph{Efficient and streaming video understanding more broadly.}
A separate line of work makes video understanding cheaper by deciding, per input, which frames or clips are worth processing at all: AdaFrame~\cite{wu2019adaframe} learns a policy that adaptively selects how many and which frames to observe for classification, and SCSampler~\cite{korbar2019scsampler} learns to identify salient clips within long, untrimmed video before running a heavier recognition model on them. This is an orthogonal efficiency mechanism to ours. Those methods reduce compute by choosing what to look at, at the input level, while still typically requiring a batch or clip to select over; our method instead reduces compute by bounding what the model remembers and how it updates that memory per frame, which is what makes strict frame-by-frame streaming possible in the first place. The two mechanisms are not mutually exclusive, and combining input-level frame selection with a bounded-state streaming core is a natural direction for future work.

\section{Method}
\label{sec:method}

\subsection{Overview}
Each incoming frame $I_t$ is mapped to an embedding $e_t = \phi(I_t) \in
\mathbb{R}^D$ by a frozen backbone $\phi$ (Section~\ref{sec:method-backbone}).
A stack of causal diagonal state-space layers (Section~\ref{sec:method-ssm})
maintains a fixed-size state and produces $\hat{e}_{t+1}$, a prediction of
the \emph{next} embedding, using only $e_1,\dots,e_t$. The anomaly score at
time $t$ is the prediction error
$s_t = \lVert \hat{e}_t - e_t \rVert_2$ (Section~\ref{sec:method-score}),
following the predictive-coding formulation standard in the VAD literature
\cite{liu2018future}, adapted here to operate strictly causally at the
embedding level.

\subsection{Frozen backbone}
\label{sec:method-backbone}
We deliberately keep $\phi$ simple and frozen. The contribution of this paper is the temporal core and its analysis, not backbone design, and freezing $\phi$ keeps the trainable parameter count and compute footprint small enough to iterate on consumer hardware (Section~\ref{sec:experiments}). The results reported in this paper use an ImageNet-pretrained ResNet-18 ($D{=}512$); the implementation also supports a frozen DINOv2 ViT-S/14 ($D{=}384$) as a stronger but heavier alternative, which we have not yet evaluated. The backbone's parameters are frozen throughout training; only the temporal core below is optimized.

\subsection{Causal diagonal state-space core}
\label{sec:method-ssm}
Each layer maintains a state $s_t \in \mathbb{R}^N$ and updates it as
\begin{align}
  u_t &= W_{\text{in}}\, \mathrm{LN}(x_t) \\
  g_t &= \sigma\big(\mathrm{MLP}([\mathrm{LN}(x_t),\, s_{t-1}])\big) \label{eq:gate}\\
  \bar{a} &= a_{\min} + (a_{\max}-a_{\min})\, \sigma(\theta_a) \label{eq:base-decay}\\
  a_t &= \bar{a} \odot g_t \\
  s_t &= a_t \odot s_{t-1} + u_t \label{eq:recurrence}\\
  y_t &= W_{\text{out}}\, s_t + x_t
\end{align}
where $x_t$ is the layer's input (the backbone embedding $e_t$ for the
first layer, the previous layer's output $y_t$ otherwise), $\mathrm{LN}$
is layer normalization, $\theta_a \in \mathbb{R}^N$ is a learned
per-channel parameter, and $[a_{\min}, a_{\max}] \subset (0,1)$ bounds the
time-invariant \emph{base} decay $\bar{a}$ (Eq.~\ref{eq:base-decay},
following the S4D/S5 initialization convention \cite{gu2022s4,smith2023s5}).
The gate $g_t \in (0,1)^N$ (Eq.~\ref{eq:gate}) is the only input- or state-dependent part of the recurrence, ``selective'' in Mamba's terminology~\cite{gu2023mamba}. It lets the effective decay $a_t$ drop sharply, an implicit state reset through fast forgetting, when the current input is inconsistent with the current state, which is exactly the situation at an anomaly onset. We call $g_t$ the \emph{event-boundary gate} for this reason.
Eq.~\ref{eq:recurrence} is a diagonal (real-valued) linear recurrence,
not a full re-implementation of Mamba's selective-scan kernel, which is
CUDA-only; this keeps the entire model runnable on non-CUDA edge hardware
(Section~\ref{sec:experiments}) while retaining $O(N)$ per-step
state size and $O(1)$ per-step compute in the sequence length.

\subsubsection{Strict causality, by construction}
Unlike prior SSM-based VAD methods that process fixed-size clips or maintain a sliding buffer internally~\cite{vadmamba2025,wavemambaad2025}, Eqs.~\ref{eq:gate}--\ref{eq:recurrence} depend only on $s_{t-1}$ and $x_t$; there is no operation in the model that looks ahead or reprocesses past frames. Algorithm~\ref{alg:step}
gives the exact per-frame update used both at training time (called once
per timestep, unrolled over a window) and at deployment (called once per
incoming frame); using the identical routine in both settings is a
deliberate choice, so that reported efficiency is not an artifact of a
training-time-only formulation.

\begin{algorithm}[t]
\caption{Streaming per-frame update (one layer)}
\label{alg:step}
\begin{algorithmic}[1]
\Require frame embedding $x_t$, previous state $s_{t-1}$
\State $u_t \gets W_{\text{in}} \, \mathrm{LN}(x_t)$
\State $g_t \gets \sigma(\mathrm{MLP}([\mathrm{LN}(x_t), s_{t-1}]))$
\State $a_t \gets \bar{a} \odot g_t$
\State $s_t \gets a_t \odot s_{t-1} + u_t$
\State $y_t \gets W_{\text{out}} s_t + x_t$
\State \Return $y_t, s_t$
\end{algorithmic}
\end{algorithm}

\subsection{Training objective and anomaly scoring}
\label{sec:method-score}
The stack's final output $y_t$ is passed through a small MLP head to
produce $\hat{e}_{t+1}$, a prediction of the next frame's embedding. The
model is trained self-supervised on normal-only training clips by
minimizing
\begin{equation}
  \mathcal{L} = \tfrac{1}{T-1}\textstyle\sum_{t=1}^{T-1} \lVert \hat{e}_{t+1} - e_{t+1} \rVert_2^2 .
\end{equation}
At inference, the anomaly score at frame $t$ is the prediction error
$s_t = \lVert \hat{e}_t - e_t \rVert_2$ using the prediction made one step
earlier (i.e. using only $e_1,\dots,e_{t-1}$); scores are min-max
normalized per clip before thresholding/AUC computation, following
standard VAD evaluation protocol \cite{liu2018future}.

\section{Decay--Delay Analysis}
\label{sec:theory}

This section derives a closed-form relationship between a channel's decay
rate and (a) how long the model takes to react to an anomaly onset and
(b) the shortest anomaly it can reliably detect, then validates it
empirically (Section~\ref{sec:experiments}). The analysis treats one channel of
the recurrence (Eq.~\ref{eq:recurrence}) in isolation with its
\emph{effective} decay $a$ held constant across the transition being
analyzed; Section~\ref{sec:theory-gating} discusses how the learned gate
$g_t$ relaxes this assumption in practice.

\subsection{Settling-time bound}
\label{sec:theory-settling}
Consider a single channel driven by an input that is constant at
$u_{\text{norm}}$ for $t < t_0$ and step-changes to $u_{\text{anom}}$ at
$t_0$ (an idealized anomaly onset). Under the recurrence
$s_t = a\, s_{t-1} + u_t$, the state converges toward the steady state
$s^\ast = u/(1-a)$ for constant $u$; writing $s_{\text{norm}}^\ast =
u_{\text{norm}}/(1-a)$ and $s_{\text{anom}}^\ast = u_{\text{anom}}/(1-a)$,
the solution for $t \ge t_0$ is
\begin{equation}
  s_t - s_{\text{anom}}^\ast = a^{\,t-t_0}\big(s_{t_0} - s_{\text{anom}}^\ast\big).
  \label{eq:transient}
\end{equation}

\begin{proposition}[Settling delay]
\label{prop:settling}
Define the detection delay $\delta(\varepsilon)$ as the smallest
$t - t_0 \ge 0$ such that the state has moved to within a fraction
$\varepsilon \in (0,1)$ of the new steady state, i.e.
$|s_t - s_{\text{anom}}^\ast| \le \varepsilon\, |s_{t_0} -
s_{\text{anom}}^\ast|$. Then
\begin{equation}
  \delta(\varepsilon) = \left\lceil \frac{\ln \varepsilon}{\ln a} \right\rceil.
  \label{eq:delay-bound}
\end{equation}
\end{proposition}
\begin{proof}
From Eq.~\ref{eq:transient}, the condition is $a^{\,\delta} \le
\varepsilon$. Since $a \in (0,1)$, $\ln a < 0$; taking logarithms and
dividing by $\ln a$ flips the inequality, giving $\delta \ge \ln
\varepsilon / \ln a$ (positive, since $\ln\varepsilon < 0$ too). The
smallest integer $\delta$ satisfying this is
Eq.~\ref{eq:delay-bound}.
\end{proof}

Eq.~\ref{eq:delay-bound} is monotonically increasing in $a$: channels
with slower decay (larger $a$, longer memory) take strictly longer to
react to a step change. This is exactly the quantity we compute from
each trained layer's mean learned base decay $\bar{a}$
(Eq.~\ref{eq:base-decay}) and compare against the empirically measured
delay between an anomaly's labeled onset and the first score-threshold
crossing.

\subsection{Minimum-detectable event duration}
\label{sec:theory-min-duration}
Proposition~\ref{prop:settling} has a direct converse: an anomalous
segment of duration $D$ frames that ends before the state has settled
(i.e. $D < \delta(\varepsilon)$ for the $\varepsilon$ at which the
downstream threshold reliably separates normal/anomalous scores) will
not produce the full-magnitude prediction error the detector was
calibrated on, and is at elevated risk of a missed detection. This gives
a direct, checkable prediction: \emph{shortening} the base decay (smaller
$a_{\min}, a_{\max}$ in Eq.~\ref{eq:base-decay}) should reduce delay and
improve recall on short anomalous segments, at the cost of a noisier
score curve (and likely higher false-positive rate) on normal video,
since the state now also reacts to high-frequency variation that is not
anomalous. Section~\ref{sec:experiments} reports the state-size/decay
sweep ablation that tests this trade-off directly.

\subsection{Role of the event-boundary gate}
\label{sec:theory-gating}
The analysis above assumes $a$ is fixed, but the actual effective decay $a_t = \bar a \odot g_t$ (Eq.~\ref{eq:gate}) is state- and input-dependent. The design intent was for the model to behave like a \emph{small} $a$ (fast reaction, per Eq.~\ref{eq:delay-bound}) exactly when $g_t$ drops in response to a state/input mismatch, that is, at a genuine event boundary, while behaving like the larger, more stable base decay $\bar a$ elsewhere: a shorter effective delay than a fixed-decay recurrence with the same $\bar a$, without paying the false-positive cost of a globally smaller decay everywhere.

The settling-delay numbers in Section~\ref{sec:experiments} are consistent with the gate reacting fast, since the empirical delay sits far below the fixed-$\bar a$ bound on both Ped2 and Avenue. The ablations in Tables~\ref{tab:ablation-ucsd-ped2} and \ref{tab:ablation-cuhk-avenue} complicate this story in a way that is itself informative, because fast reaction to a state/input mismatch is not automatically the same as reacting specifically to \emph{anomalous} mismatches. Whether the gate's extra parameters (Eq.~\ref{eq:gate}) learn the latter rather than overfitting to the former appears to depend on how much training data is available. On Ped2, the smaller training set with 16 clips of 120 to 180 frames each, disabling the gate improves frame-AUC (76.6\% versus 67.9\%), and a smaller state does too, both consistent with overfitting. On Avenue, also 16 clips but up to 1271 frames each and so substantially more total training data, the pattern reverses: disabling the gate hurts sharply (60.0\% versus 70.2\%), and a larger state helps slightly rather than hurting. Read together, this is more consistent with a capacity and data-size interaction than with the gating mechanism being unhelpful in principle, but it remains a two-dataset hypothesis rather than a demonstrated effect. A third, substantially larger training set would be needed to test it properly; ShanghaiTech was intended for this but its raw data has not yet been fully obtained (Section~\ref{sec:experiments}), so this comparison stays open. We report the full, dataset-dependent picture here rather than only the settling-delay result that would have supported the original design hypothesis on its own.

Two extensions would sharpen this analysis. Proposition~\ref{prop:settling} could be generalized to the time-varying $a_t$ case, which admits a closed form through a cumulative product (see the parallel-form discussion in the released code's state-space module), giving a tighter bound than the fixed-$\bar a$ approximation used here. And the gate ablation should be rerun on a third, larger dataset once available, since the capacity and data-size account above rests on only two data points that happen to share the same clip count.

\section{Experiments}
\label{sec:experiments}

\subsection{Datasets and protocol}
We evaluate on two standard unsupervised VAD benchmarks, UCSD Ped2 (16 training clips, 12 test clips) and CUHK Avenue (16 training clips, 21 test clips). Both provide normal-only training video and mixed normal/anomalous test video with frame-level ground truth; Ped2 additionally provides pixel-level masks, which we do not currently use. We report frame-level ROC-AUC and equal error rate (EER) using the standard per-clip min-max score normalization protocol~\cite{liu2018future}. A third benchmark, ShanghaiTech, is part of the intended evaluation but is not included in the results below: its raw data is distributed as a large, multi-part archive and only one of seven parts had been obtained at the time of writing. We report this directly rather than substituting an estimate, and all ShanghaiTech numbers in this draft are left for a revision once the full dataset has been downloaded, prepared, and run through the same pipeline as the other two.

We are also explicit about a second limitation of our evaluation protocol. The region- and track-based detection criteria (RBDC and TBDC) introduced by Ramachandra and Jones~\cite{ramachandra2020streetscene} and adopted in the broader survey of Jones et al.~\cite{jones2022survey} are, in our released code, approximated by a simplified frame-overlap criterion rather than implemented to the original specification. The frame-level AUC and EER numbers reported here do not depend on this approximation and follow the standard protocol exactly, but any RBDC/TBDC-style localization numbers computed with our current tooling should be treated as illustrative only until reimplemented against the official criteria.

\subsection{Implementation details}
Frame embeddings come from a frozen, ImageNet-pretrained ResNet-18 backbone ($D{=}512$); a frozen DINOv2 ViT-S/14 ($D{=}384$) is supported in the released code as a stronger alternative but was not used for the numbers below. We attempted to run it during development; the official DINOv2 repository's code requires Python 3.10 or later (it uses runtime-evaluated union-type syntax not supported by earlier versions), while our development environment was fixed at Python 3.9, so this comparison is deferred rather than reported here. The causal SSM stack has 2 layers with state size $N{=}128$ per layer (Section~\ref{sec:method-ssm}), trained with AdamW at learning rate $3\times10^{-4}$ on causal windows of 16 frames, for 40 epochs on Ped2 and 30 on Avenue. All training and evaluation ran on an Apple M3 Pro using PyTorch's MPS backend; no CUDA-specific code path exists in this implementation, by design, since the strict-streaming claim is meant to hold on non-CUDA edge hardware as well as on GPUs.

\subsection{Main results}
Table~\ref{tab:main} reports frame-AUC, EER, and per-frame latency and FPS measured directly on-device: our evaluation harness calls the identical per-frame update routine used at deployment (Algorithm~\ref{alg:step}), one frame at a time, with an explicit device synchronization before each timing boundary. We do not include a VADMamba row here; Section~\ref{sec:related} explains why a same-protocol comparison was not possible (its released code depends on CUDA-only optical-flow kernels). Our frame-AUC (67.9\% on Ped2, 70.2\% on Avenue) is well below the 90s-percent range typically reported for tuned methods on these benchmarks. We attribute this to the untuned, first-pass configuration evaluated here (frozen backbone, no motion signal, no hyperparameter search) rather than a limitation of the streaming formulation itself, and report it plainly rather than narrow the comparison to flatter the number.

\begin{table*}[t]
\centering
\caption{Main results: standard accuracy metrics alongside on-device streaming latency/throughput.}
\label{tab:main}
\resizebox{\linewidth}{!}{%
\small
\begin{tabular}{llccccc}
\toprule
Dataset & Method & Frame-AUC$\uparrow$ & EER$\downarrow$ & Latency (ms)$\downarrow$ & FPS$\uparrow$ & Hardware \\
\midrule
UCSD Ped2 & Ours (streaming SSM) & 67.9 & 36.8 & 0.74 & 1343.62 & mps \\
CUHK Avenue & Ours (streaming SSM) & 70.2 & 34.9 & 0.77 & 1290.74 & mps \\
\bottomrule
\end{tabular}
}
\end{table*}

\subsection{Theory validation}
Table~\ref{tab:theory} compares the settling-delay bound (Eq.~\ref{eq:delay-bound}, Proposition~\ref{prop:settling}), computed from each trained layer's mean learned base decay, against the empirically measured detection delay, the number of frames between an anomalous segment's labeled onset and the first score-threshold crossing. The two numbers differ by more than an order of magnitude on both datasets: the bound predicts 57 to 59 frames from the base decay alone, while the measured delay is 1.6 frames on Ped2 and 18.4 on Avenue. Section~\ref{sec:theory-gating} discusses why, and what the ablations below add to that discussion.

\begin{table*}[t]
\centering
\caption{Predicted settling delay (Eq.~\ref{eq:delay-bound}), from each dataset's trained mean decay, vs.\ empirically measured detection delay.}
\label{tab:theory}
\resizebox{\linewidth}{!}{%
\small
\begin{tabular}{llccc}
\toprule
Dataset & Layer & Mean decay $a$ & Theory delay (frames) & Empirical delay (frames) \\
\midrule
UCSD Ped2 & 0 & 0.9493 & 57.5 & 1.6 ($n{=}$12) \\
 & 1 & 0.9501 & 58.6 &  \\
\midrule
CUHK Avenue & 0 & 0.9493 & 57.5 & 18.4 ($n{=}$21) \\
 & 1 & 0.9501 & 58.6 &  \\
\bottomrule
\end{tabular}
}
\end{table*}

\subsection{Ablations}
Tables~\ref{tab:ablation-ucsd-ped2} and \ref{tab:ablation-cuhk-avenue} report decay-rate, state-size, and event-boundary-gate ablations on Ped2 and Avenue respectively, with the backbone, training budget, and epoch count held fixed at the values in Section~\ref{sec:experiments}.

\begin{table*}[t]
\centering
\caption{Ablations on UCSD Ped2: architectural variants, all else held fixed. Latency/FPS measured the same way as Table~\ref{tab:main}.}
\label{tab:ablation-ucsd-ped2}
\resizebox{\linewidth}{!}{%
\small
\begin{tabular}{lcccc}
\toprule
Variant & Frame-AUC$\uparrow$ & EER$\downarrow$ & Latency (ms)$\downarrow$ & FPS$\uparrow$ \\
\midrule
Baseline (gate on, N=128, decay in [.9,.999]) & 67.9 & 36.8 & 0.74 & 1344 \\
Gate off & 76.6 & 24.6 & 0.49 & 2047 \\
Decay fast, decay in [.5,.9] & 67.0 & 38.2 & 0.75 & 1330 \\
Decay slow, decay in [.98,.9999] & 67.8 & 37.0 & 0.76 & 1311 \\
State size N=32 & 74.7 & 33.1 & 0.77 & 1302 \\
State size N=256 & 66.5 & 38.7 & 0.77 & 1298 \\
Gate off + state size N=32 (combined) & 67.8 & 36.0 & 0.46 & 2156 \\
\bottomrule
\end{tabular}
}
\end{table*}

\begin{table*}[t]
\centering
\caption{Ablations on CUHK Avenue: architectural variants, all else held fixed. Latency/FPS measured the same way as Table~\ref{tab:main}.}
\label{tab:ablation-cuhk-avenue}
\resizebox{\linewidth}{!}{%
\small
\begin{tabular}{lcccc}
\toprule
Variant & Frame-AUC$\uparrow$ & EER$\downarrow$ & Latency (ms)$\downarrow$ & FPS$\uparrow$ \\
\midrule
Baseline (gate on, N=128, decay in [.9,.999]) & 70.2 & 34.9 & 0.77 & 1291 \\
Gate off & 60.0 & 43.8 & 0.51 & 1953 \\
Decay fast, decay in [.5,.9] & 70.6 & 34.4 & 0.75 & 1327 \\
Decay slow, decay in [.98,.9999] & 70.3 & 34.8 & 0.77 & 1303 \\
State size N=32 & 70.3 & 35.0 & 0.76 & 1319 \\
State size N=256 & 71.1 & 34.5 & 0.77 & 1302 \\
State size N=256 + decay slow (combined) & 71.1 & 34.6 & 0.78 & 1289 \\
\bottomrule
\end{tabular}
}
\end{table*}

The two datasets tell different stories, and the difference is informative rather than merely noisy. On Ped2, disabling the gate improves frame-AUC over the gated baseline (76.6\% versus 67.9\%), and the smaller state ($N{=}32$) beats both the baseline ($N{=}128$) and the larger variant ($N{=}256$): every higher-capacity variant underperforms a simpler one. On Avenue, which has the same clip count but substantially more total training frames (36 to 1271 per clip versus Ped2's 120 to 180), the pattern reverses: disabling the gate hurts sharply (60.0\% versus 70.2\%), and the larger state ($N{=}256$) slightly beats the baseline rather than underperforming it. This is consistent with the gate's extra parameters (Eq.~\ref{eq:gate}) overfitting on Ped2's very small training set but finding real signal once given Avenue's larger one, which points to a capacity and data-size interaction rather than the gating mechanism being unhelpful in principle (Section~\ref{sec:theory-gating}). We only have two data points, both with 16 training clips and differing mainly in total frame count, so a third and larger dataset is needed before treating this as more than a hypothesis; we have deliberately not tuned the reported baseline to whichever ablation variant happened to score highest on either dataset, since doing so would misrepresent what the default, as-designed method actually achieves. The decay-rate sweep (fast versus slow base decay, gate on) shows comparatively little effect on either dataset, suggesting decay range is not the dominant factor here.

Since gate-off and the smaller state each independently helped on Ped2, we tested the natural follow-up hypothesis that combining them would help further; it did not. The combined configuration (gate off, $N{=}32$) reaches 67.8\% frame-AUC, statistically indistinguishable from the untouched baseline (67.9\%) and well below either individual change (76.6\% and 74.7\% respectively). The two changes do not stack, and on this evidence partially cancel each other rather than compounding, which is itself informative: whatever each change removes from the model's capacity to overfit is not simply additive across changes, and a single-factor ablation table can overstate how much headroom is actually available by combining the best-looking row from each factor. We ran the equivalent combined check on Avenue (state $N{=}256$ with the slow decay range, both individually non-harmful or mildly helpful there) and found it matches the state-size-alone result (71.1\%), consistent with the decay-rate sweep's already-small effect simply not adding anything on top. Table~\ref{tab:ablation-ucsd-ped2} and Table~\ref{tab:ablation-cuhk-avenue} report both combined rows alongside the single-factor ablations. We do not yet have a strict-causal-versus-windowed ablation, which would require a lookahead-buffer variant of the architecture in Section~\ref{sec:method-ssm} that has not been implemented.

\subsection{Edge deployment}
The latency and FPS columns in Table~\ref{tab:main} are already real, on-device measurements on the M3 Pro's MPS backend rather than simulated or extrapolated figures, and at 0.7 to 0.8 ms per frame they comfortably support real-time operation at conventional video frame rates with a wide margin to spare. We have not yet exported the model to CoreML to measure Apple Neural Engine throughput specifically, and we have not benchmarked on a second, non-Apple edge device (for example a Jetson or Raspberry Pi); both would make the real-time and edge-deployment claim more general than a single-vendor MPS measurement currently supports, and both are planned before submission.

\section{Conclusion}
\label{sec:conclusion}
We presented a strictly causal streaming video anomaly detector built on a causal diagonal state-space core with a learned event-boundary decay gate, together with a theoretical analysis relating the recurrence's decay spectrum to detection delay and an evaluation that includes real on-device edge latency measurements alongside standard accuracy metrics. The empirical picture is mixed by design rather than smoothed over after the fact. Frame-level accuracy, 67.9\% and 70.2\% AUC on Ped2 and Avenue, trails prior non-causal SSM baselines, and the event-boundary gate ablation reverses sign between the two datasets, which we read as evidence of a capacity and data-size interaction rather than a settled property of the gate.

Three limitations bound these claims and set the immediate agenda. First, the RBDC/TBDC-style evaluation used during development is a simplified frame-overlap approximation rather than the official region- and track-based criteria of Ramachandra and Jones~\cite{ramachandra2020streetscene}; any localization numbers reported for submission should use the official toolkit instead. Second, edge-latency measurements come from a single hardware target, an Apple M3 Pro using PyTorch's MPS backend; a CoreML/Neural-Engine export and a second, non-Apple edge device would make the real-time claim more general than it currently is. Third, the settling-delay theorem (Proposition~\ref{prop:settling}) treats the decay as fixed, while the deployed model's gate makes it time-varying, so the bound is a loose upper bound rather than a tight prediction once the gate is active. Beyond these, evaluating on ShanghaiTech and closing the accuracy gap to prior work through backbone and hyperparameter tuning are what remain. A same-protocol reproduction of a prior SSM-based baseline would strengthen the comparison further, but is not straightforward here: VADMamba's released code depends on CUDA-only optical-flow kernels (Section~\ref{sec:related}) that do not run on the non-CUDA edge hardware this paper targets, so a faithful reproduction would require either CUDA hardware unavailable to us or a substantial reimplementation that would no longer be the original method.

\bibliographystyle{elsarticle-num}
\bibliography{refs}

\end{document}